\documentclass[letterpaper, 10 pt, conference]{ieeeconf}  

\IEEEoverridecommandlockouts                              

\usepackage[english]{babel}

\usepackage{amsmath,amssymb}
\usepackage{mathtools}
\usepackage{graphicx}

\usepackage[colorlinks=true, allcolors=blue]{hyperref}
\usepackage{subcaption}

\usepackage{tikz}
\usetikzlibrary{automata,positioning}
\usepackage{acronym}

\usepackage{tabularx}
\usepackage{amsthm}
\newif\ifuseboldmathops
\newif\ifuseittextabbrevs
\useboldmathopstrue   

\ifuseittextabbrevs

	\newcommand{\ie}{{\it i.e.}}

\else

	\newcommand{\ie}{i.e.}

\fi

\ifuseboldmathops

\else

\fi

\ifuseboldmathops

\else

\fi

\ifuseboldmathops

\else

\fi

\ifuseboldmathops

\else

\fi

\newcommand{\argmax}{\mathop{\mathrm{argmax}}}

\newcommand{\abs}[1]{\lvert#1\rvert}

\newcommand{\calAP}{\mathcal{AP}}
\newcommand{\calPA}{\mathcal{A}}
\newcommand{\calG}{\mathcal{G}}

\newcommand{\calZ}{\mathcal{Z}}

\newcommand{\init}{{\iota}}

\acrodef{mdp}[MDP]{Markov Decision Process}
\acrodef{pomdp}[POMDP]{Partially Observable Markov Decision Process}
\acrodef{momdp}[MOMDP]{Multi-objective MDP}
\acrodef{ltl}[LTL]{Linear TtlemporTLal LoTeminating Labeled gic}
\acrodef{dfa}[DFA]{Deterministic Finite Automaton}
\acrodef{tlmdp}[TLMDP]{terminating labeled Markov decision process}
\acrodef{pdfa}[PDFA]{preference deterministic finite automaton}

\theoremstyle{definition}
 \newtheorem{definition}{Definition}
 \newtheorem{example}{Example}
\newtheorem{problem}{Problem}
\newtheorem{lemma}{Lemma}
\newtheorem{assumption}{Assumption}

\newtheorem{theorem}{Theorem}

\newcommand{\calA}{\mathcal{A}}

\newcommand{\augnodes}{\mathcal{W}}

\newcommand{\augedges}{\mathcal{E}}

\acrodef{smdp}[Semi-MDP]{Semi-Markov decision process}
\acrodef{mcts}[MCTS]{Monte Carlo tree search}
\acrodef{uct}[UCT]{Upper Confidence Bound 1 applied to trees}
\acrodef{scltl}[scLTL]{syntactically co-safe LTL}
\acrodef{ssp}[SSP]{Stochastic Shortest Path}
\acrodef{p2sg}[SG(2)]{Two-player Stochastic Game}
\acrodef{mc}[MC]{Markov chain}
\acrodef{prefltl}[TPL]{ Temporal Preference Logic}
\acrodef{tld}[TLwD]{Temporal Logic with Distributions}
\acrodef{mtl}[Metric TL]{Metric Temporal Logic}
\acrodef{sta}[STA]{Stochastic Timed Automaton}

\newcommand{\dist}{\mathcal{D}}

\renewcommand{\Pr}{\mathbf{Pr}}

\newcommand{\calM}{\mathcal{M}}
\newcommand{\calP}{\mathcal{P}}

\newcommand{\reach}[1]{\mathsf{reach}(#1)}

\acrodef{gpf}[GPF]{generalized preference formula}

\acrodef{cp}[CP]{ceteris paribus}
\acrodef{milp}[MILP]{Mixed-Integer Linear Programming}
\acrodef{dfa}[DFA]{Deterministic Finite Automaton}

\newcommand{\pdfa}{\mathcal{A}}

\newcommand{\prefvertices}{\mathbb{F}}

\newcommand{\prefedges}{E}

\newcommand{\Paths}{\operatorname{Paths}}

\newcommand{\trace}{\operatorname{trace}}

\newcommand{\PPwPOP}{{\rm PPwPOP}\xspace}

\newcommand*{\probleminternal}[4]{
	\par
	\medskip
	\noindent\fbox{\parbox{0.98\columnwidth}{
			\textbf{#4: #1} \\[0.05in]
			\renewcommand{\tabcolsep}{2pt}
			\begin{tabularx}{\linewidth}{rX}
				\emph{Input:} & #2 \\
				\emph{Output:} & #3
			\end{tabularx}
		}}
		\par
		\medskip
		\par
	}
	
\newcommand*{\problembox}[3]{\probleminternal{#1}{#2}{#3}{Problem}}

\acrodef{ltlf}[LTLf]{Linear temporal logic over finite words}
\acrodef{pltlf}[PLTLf]{Preference over linear temporal logic over finite words}

\usepackage{xcolor}

\usepackage{framed}
\usepackage{changes}
\definechangesauthor[name=Jie Fu, color=blue]{JF}
\usepackage{textcomp}
\usepackage{algorithmic}
\usepackage[ruled,vlined,noend,linesnumbered]{algorithm2e}
\usepackage[english]{babel}
\usepackage{enumerate}
\usepackage{mathrsfs}

\newcommand*{\gobble}[1]{}

\usepackage{cite}

\usepackage[framemethod=default,innerleftmargin=5pt,innerrightmargin=5pt,skipabove=5pt,topline=false,rightline=false,bottomline=false,linewidth=2.5pt,roundcorner=5pt]{mdframed}

\title{\LARGE \bf
 Probabilistic Planning with Partially Ordered Preferences over Temporal Goals
}

\author{Hazhar Rahmani, Abhishek N. Kulkarni, and Jie Fu  
\thanks{$^{1}$The authors are with the Department of Electrical and Computer Engineering,
        University of Florida,  Gainesville, FL 32605, USA.
        {\tt\small \{h.rahmani, a.kulkarni2, fujie\}@ufl.edu }%
        This material is based upon work supported by the Air Force Office of Scientific Research under award
number FA9550-21-1-0085 and in part by NSF under award number 2024802.
}
}

\begin{document}
	\maketitle
	
	\begin{abstract}
		In this paper, we study planning in stochastic systems, modeled as Markov decision processes (MDPs), with preferences over temporally extended goals.
		Prior work on temporal planning with preferences assumes that the user preferences form a total order, meaning that every
		pair of outcomes are comparable with each other.
        In this work, we consider the case where the preferences over possible outcomes are a partial order rather than a total order.
	We first introduce a variant of deterministic finite automaton, referred to as a preference DFA, for specifying the user's preferences over temporally extended goals.
	Based on the order theory, we translate the preference DFA to a preference relation over policies for probabilistic planning in a labeled MDP.   
    In this treatment, a most preferred policy induces a \emph{weak-stochastic nondominated} probability distribution over the finite paths in the MDP.
    The proposed planning algorithm hinges on the construction of a multi-objective MDP. 
    We prove that a weak-stochastic nondominated policy given the preference specification is Pareto-optimal in the constructed multi-objective MDP, and vice versa.
		Throughout the paper, we employ a running example to  demonstrate the proposed preference specification and solution approaches. We show  the efficacy of our algorithm using the example with detailed analysis, and then discuss possible future directions.	
	\end{abstract}

	\section{Introduction}
	\label{sec:intr}
	
	With the rise of artificial intelligence, robotics and autonomous systems are being designed to make complex decisions by reasoning about multiple goals at the same time. 
	Preference-based planning (PBP) allows the systems to decide which goals to satisfy when not all of them can be achieved \cite{hastie2010rational}. 
	Even though PBP has been studied since the early 1950's, most works on preference-based temporal planning (c.f. \cite{baier2008planning}) assume that all outcomes are pairwise comparable---that is, the preference relation is a \emph{total} order. 
	This assumption is strong and, in many cases, unrealistic\cite{aumann1962utility}. 
	In robotic applications, preferences may need to admit a \emph{partial} order because of (a) \emph{Inescapability}: An agent has to make decisions under time limits but with partial information about preferences because, for example, it lost communication with the server; and (b) \emph{Incommensurability}: Some situations, for instance, comparing the quality of an apple to that of banana, are fundamentally incomparable since they lack a standard basis to compare. These situations motivate the need for a planner that deals with partial order preferences in the presence of all uncertainties in its environment.
	
	\begin{figure}[h]
		\centering
		\includegraphics[width=1.0\linewidth]{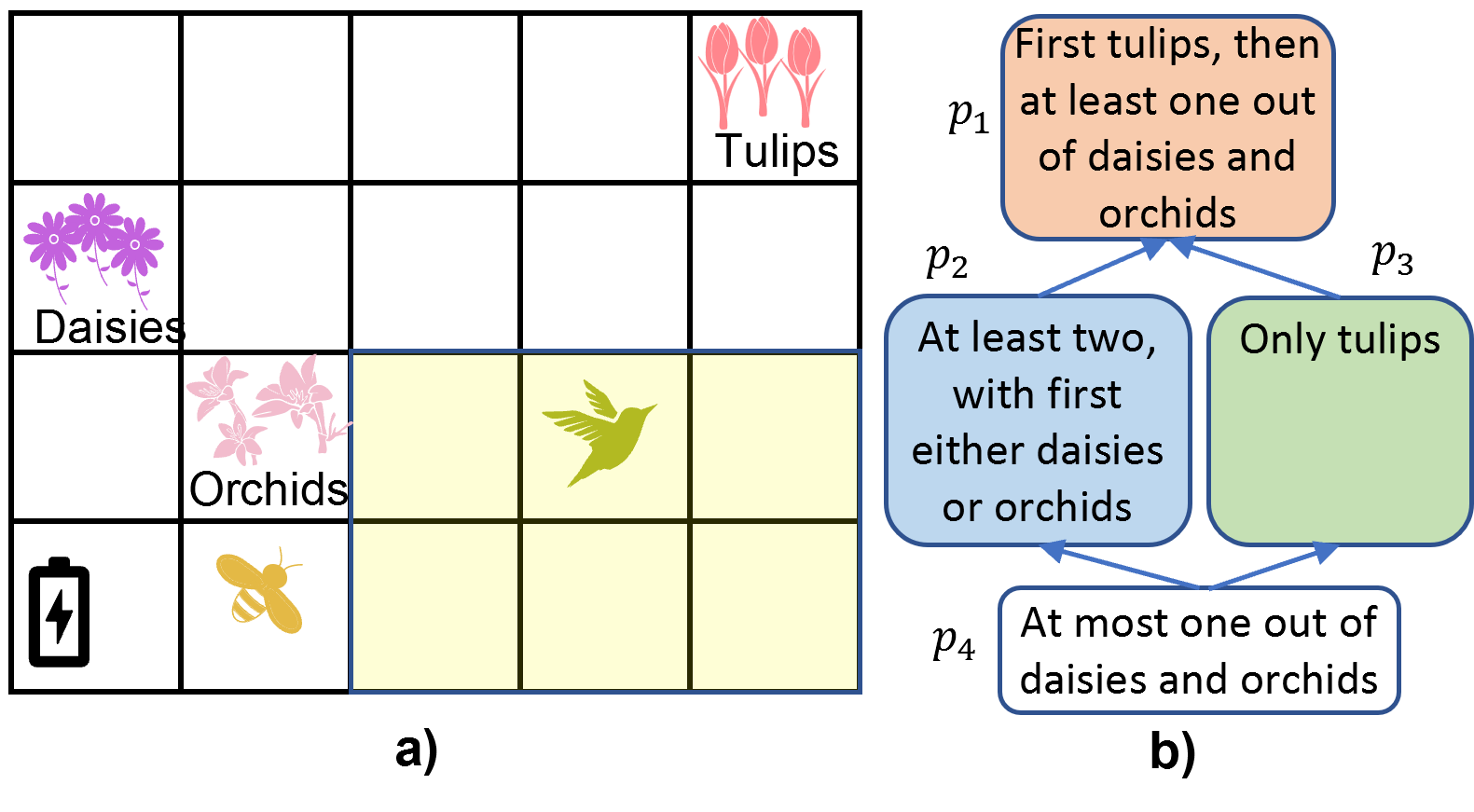}
		\caption{
		 \textbf{a)} Bob's Garden.
		 \textbf{b)} Bob's preferences on how the bee robot should perform the task of pollinating the flowers.
		}
		\label{fig:gap_garden}
	\end{figure}
	
	As a motivation example, consider Figure~\ref{fig:gap_garden}, which shows a garden that belongs to Bob. He grows three kinds of flowers: Tulips, daisies, and orchids. To pollinate the flowers, he uses a bee robot with limited battery. The environment is uncertain due to the presence of another agent (bird), the weather, and the robot dynamics.
    
    Bob has a preference for how the robot should achieve the task of pollination. Compared to the other types, tulips have a shorter life span, so Bob considers four outcomes 
    \begin{itemize}
        \item[($p_1$)] pollinate tulips first, then at least one other flower type; 
        \item[($p_2$)] pollinate two types of flowers, with the first being either daisies or orchids; 
        \item[($p_3$)] pollinate only tulips; and
        \item[($p_4$)] at most one out of daisies and orchids is pollinated,
    \end{itemize}
    where the preference relation among them is shown in Figure~\ref{fig:gap_garden}b using a preference graph, where the nodes represent the outcomes, and each directed edge is an improving flip \cite{santhanamRepresentingReasoningQualitative2016}. Thus, $p_1$ is the most preferred and $p_4$ is the least preferred outcome, while $p_2$ and $p_3$ are incomparable with each other. 
    %
    %
    As the robot has a limited battery life and the system is stochastic, it might not achieve the most preferred outcome with probability one. 
     Incomparable outcomes also introduce incomparable policies.   

    
   Preference-based planning problems over temporal goals have been well-studied for deterministic systems given both total and partial preferences (see \cite{baier2008planning} for a survey). For preferences over temporal goals in deterministic systems, several works \cite{tumova2013least, wongpiromsarn2021, rahmani2020what} proposed minimum-violation planning methods that decide which low-priority constraints should be violated. 
   Mehdipour \emph{et al.}~\cite{mehdipourSpecifyingUserPreferences2021} associate weights with Boolean and temporal operators in signal temporal logic to specify the importance of satisfying the sub-formula and priority in the timing of satisfaction. This reduces the PBP problem to that of maximizing the weighted satisfaction in deterministic dynamical systems. 
    However, the solutions to PBP problem for deterministic systems cannot be applied to stochastic systems. 
    This is because in stochastic systems, even a deterministic policy yields a distribution over outcomes. Hence, to determine a better policy, we need comparison of distributions---a task a deterministic planner cannot do.
    
	Several works have studied the PBP problem for stochastic systems.
	Lahijanian and Kwiatkowska~\cite{Lahijanian2016} considered the problem of revising a given specification to improve the probability of satisfaction of the specification. They
	formulated the problem as a multi-objective \ac{mdp} problem that trades off minimizing the cost of revision and maximizing the probability of satisfying the revised formula. 
    Cai \emph{et al.}~\cite{cai2021optimal} consider planning with infeasible LTL specifications in systems modeled by probabilistic MDPs. Their problem's aim is to synthesize a policy that in decreasing order of importance 1) provides a desired guarantee to satisfy the task, 2) satisfies the specifications as much as possible, and 3) minimizes the implementation cost of the plan.
	Li \emph{et al.}~\cite{li2020probabilistic} solve a preference-based probabilistic planning problem by reducing it to a multi-objective model checking problem. 
	However, all these works assume the preference relation to be \emph{total}. 
	To the best of our knowledge, \cite{fu2021probabilistic} is the only work that studies the problem of probabilistic planning with incomplete preferences. 
	The authors introduce the notion of the value of preference satisfaction for planning within a pre-defined finite time duration and developed a mixed-integer linear program to maximize the satisfaction value for a subset of preference relations. In comparison, our work resorts to the notion of stochastic ordering to compare policies in the stochastic system with respect to the partial order of temporal goals and allows the time horizon to be finite, but unbounded.

 
	Our contributions in this paper are three-fold: (1) We introduce a new computational model called a \emph{Preference Deterministic Finite Automaton (PDFA)}. A PDFA models a user's (possibly partial) preferences over temporally extended goals; (2) We identify the connection between the probabilistic PBP problem and stochastic orders \cite{masseyStochasticOrderingsMarkov1987}. This allows us to reduce the problem of probabilistic planning with partial preferences over temporal goals to that of finding the set of weak-stochastic nondominated policies in a product of \ac{mdp} and the PDFA. (3) We employ the property of weak-stochastic nondominated policies to design multiple objectives in the product \ac{mdp} and prove that a Pareto-optimal policy in the resulting multi-objective product \ac{mdp} is weak-stochastic nondominated respecting the preference relation. Thus,  
	the set of weak-stochastic nondominated policies can, then, be computed using any off-the-shelf solver that computes Pareto optimal policies in polynomial time.

	\section{Preliminaries and Problem Formulation}
	\label{sec:def}
	%
	\textbf{Notations} The set of all finite words over a finite alphabet $\Sigma$ is denoted $\Sigma^\ast$. The empty string $\Sigma^0$ is denoted as $\epsilon$. The set of all probability distributions over a finite set $X$ is denoted $\dist(X)$.
    Given a distribution $\mathbf{d}\in \dist(X)$, the probability of an outcome $x\in X$ is denoted $\mathbf{d}(x)$. %

	\subsection{The System and its Policy}
	%
	
    We model the system using a variant of MDP.
	
	\begin{definition}
		\label{def:labeled_mdp}
		A \ac{tlmdp}, or a terminating MDP for short, is a tuple  $M = \langle S, A:=\bigcup_{s \in S} A_s, \mathbf{P}, s_0, s_\bot, \calAP, L \rangle$ in which 
		$S$ is a finite set of states;
		$A$ is a finite set of actions, where for each state $s \in S$, $A_s$ is the set of available actions at $s$;
		$\mathbf{P}: S \times A   \rightarrow \dist(S)$ is the probabilistic transition  function, where for each $s, s' \in S$ and $a \in A$, $\mathbf{P}(s, a, s')$ is the probability that the MDP transitions to $s'$ after taking action $a$ at $s$; 
		$s_0 \in S$ is the initial state;
		%
		%
		$s_\bot \in S$ is the \emph{termination state}, which is a unique \emph{sink} state and $A_{s_\bot} = \emptyset$; 
		$\calAP$ is a finite set of atomic propositions; and
		$L: S \rightarrow 2^{\calAP} \cup \{\epsilon\}$ is a labeling function that assigns to each state $s\in S \setminus \{s_\bot\}$, the set of atomic propositions $L(s) \subseteq \calAP $ that hold in $s$. 
		Only the terminating state is labeled the empty string, i.e., $L(s)=\epsilon$ iff $s=s_\bot$.
	\end{definition}
	%
	

	Though this definition assumes a single sink state, we do not lose generality, as one can always convert any MDP with more than one sink state into an equivalent MDP that has only a single sink state by redirecting proper transitions to that sink state. \gobble{,by keeping one of those states to be the termination state and adding transitions for $a_\bot$ from all other sink states to that termination state.}
	
	The robot's interaction with the environment in a finite number $k$  of steps produces an \emph{execution} $\varrho = s_0 a_0 s_1 a_1 \cdots s_{k-1} a_{k-1} s_{k}$, where $s_0$ is the initial state and at each step $0 \leq i \leq k$, the system is at state $s_i$, the robot performs $a_i \in A_{s_i}$, and then the system transitions to state $s_{i+1}$, picked randomly based on $\mathbf{P}$ among those states for which $\mathbf{P}(s_i, a_i, .) > 0$.
	This execution produces a \emph{path} defined as $\rho=s_0 s_1 \cdots s_k \in S^*$,
	and the \emph{trace} of this path is defined as the finite word $\trace(\rho)=L(s_0) L(s_1) L(s_2) \cdots L(s_k)  \in (2^\calAP)^*$.
	%
	%
	Path $\rho$ is called \emph{terminating} if $s_k = s_\bot$.
	The set of all terminating paths in $M$ is denoted $\Paths_{\bot}(M)$.

  A policy for $M$ is a function $\pi: \mathscr{D} \rightarrow \mathscr{C}$ where it is called \emph{memoryless} if $\mathscr{D}=S$; \emph{finite-memory} if $\mathscr{D}=S^*$; \emph{deterministic} if $\mathscr{C}=A$, and \emph{randomized} if $\mathscr{C}=\dist(A)$.
	
	%

	%
In a terminating \ac{mdp}, a policy is \emph{proper} if it guarantees that the termination state $s_\bot$ will be reached with probability one \cite{bertsekas1991analysis}. 	The set of all randomized, finite-memory, proper polices for $M$ is denoted $\Pi_{prop}^M$. We  are  only interested in  finite traces for which a preference relation is defined. Thus, we only consider proper policies.
	\begin{assumption}
	  We assume all the policies for the MDP are proper.
	\end{assumption}
	In this paper, we consider only the MDPs for which all the policies are proper.
	We consider applications where the robot finishes its execution in a finite time, and in fact, in many robotics application, the robot has a battery limit or a limited lifespan and cannot execute forever.
	\subsection{Rank the policies}

	We introduce a computational model that captures the user's preference over different temporal goals. \gobble{the robot can achieve in the environment.}
	%
	\begin{definition}
		\label{def:model_preference}
		Given a countable set $U$, a \emph{preference model} for $U$, denoted $\succeq^U$ 
		is a partial order over the elements of $U$.
	\end{definition}
	We simply use  $\succeq$ for $\succeq^U$ if its meaning is clear
		from the context.
     Given $u_1, u_2 \in U$, we write $u_1 \succeq u_2$ if $u_1$ is \emph{weakly preferred to} (\ie,  is at least as good as) $u_2$; and $u_1\sim u_2$ if $u_1\succeq u_2$ and $u_2\succeq u_1$,
    that is,  $u_1$ and $u_2$ are \emph{indifferent}.
    We write
    $u_1 \succ u_2$ to mean that $u_1$ is \emph{strictly preferred} to $u_2$, \ie, $u_1\succeq u_2$ and $u_1\not \sim u_2$.  We write
	$u_1 \nparallel u_2$ if $u_1$ and $u_2$ are \emph{incomparable}.
		\begin{definition}\cite{masseyStochasticOrderingsMarkov1987}. Given a countable set $U$ partially ordered by a preference model $\succeq$, the \emph{weak-stochastic ordering} for $U$ is denoted $\mathfrak{E}_{wk}(U)$ and is defined as the family of subsets 
		   \begin{equation}
		       \mathfrak{E}_{wk}(U) = \{ \{x\}^\uparrow \mid x \in U \} \cup \{U, \emptyset\}.
		   \end{equation}
		   where $\{x\}^\uparrow =\{y \mid y \succeq x \}$ contains all elements in $U$ that are at least as good as $x$, according to the partial order $\succeq$.
		\end{definition}
        
        The weak-stochastic ordering for $U$ allows us to rank different probability measures on $U$.
        Given two probability measures $P_1$ and $P_2$ on $U$, we say $P_1$ \emph{weak-stochastic dominates} $P_2$ under $\succeq$, denoted $P_1 >_{\mathfrak{E}_{wk}} P_2$,
        if $P_1[X] \geq P_2[X]$ for each $X \in \mathfrak{E}_{wk}(U)$ and $P_1[Y] > P_2[Y]$ for some $Y \in \mathfrak{E}_{wk}(U)$. Intuitively, for any  outcome  $x$ in $U$, the probability of getting an outcome (weakly) preferred to $x$ in $P_1$ is at least as good as that in $P_2$, and for some outcome $x'\in U$, the probability of getting an outcome preferred to $x'$ in $P_1$ is higher than that in $P_2$.
        %
        
        To illustrate, consider the following example.
     \begin{example}
	 	Let $U = \{a,b,c,d\}$ and  $\succeq = \{(a, b), (b, d), (c, d), (a, c), (a, d) \}$, where $(x,y)\in \succeq$ if and only if $x\succeq y$.
	 	We have
	 		\[
	 	\mathfrak{E}_{wk}(U) = \{ \{a\}, \{a,b\}, \{a,c\}, \{a,b,c,d\}, \emptyset\}.
	 	\]
	 	Now consider three probability measures $P_1$, $P_2$, and $P_3$ where $P_1(a)=P_1(b)= 0.5$, $P_2(a) =P_2(c) = 0.5$, and $P_3(a) =P_3(d) = 0.5$.   
	 	Accordingly, 
	 	\[
	 	[P_1[X]]_{X\in	\mathfrak{E}_{wk}(U) }= [0.5, 1, 0.5, 1, 0],
	 	\]
	 	\vspace{-18pt}
	 	\[
	    [P_2[X]]_{X\in	\mathfrak{E}_{wk}(U) }= [0.5, 0.5, 1, 1, 0], \text{and}
	 	\] 
	 	\vspace{-18pt}
	 	\[
	 	[P_3[X]]_{X\in 	\mathfrak{E}_{wk}(U) }   = [0.5, 0.5, 0.5, 1, 0].
	 	\]
	 	Therefore, $P_1 >_{\mathfrak{E}_{wk}} P_3$, $P_2 >_{\mathfrak{E}_{wk}} P_3$. None of $P_1$ and $P_2$ weak-stochastic dominates the other one.
    \end{example}

	In this context, 
	the user preference over temporal goals is a preference model for $U=\Sigma^*$ where $\Sigma= 2^\calAP$. Based on the ranking of probability measures induced by the weak-stochastic ordering for $\Sigma^\ast$, we can rank the proper policies $\Pi^M_{prop}$ in the \ac{tlmdp} as follows.
	 
	Note that a proper policy $\pi: S^\ast \rightarrow \dist(A)$ produces a distribution 
	over the set of all terminating paths in the MDP $M$ such that for each terminating path $\rho \in \Paths_{\bot}(M)$,
	$\Pr^\pi(\rho)$
	is the probability of generating $\rho$ when the robot uses policy $\pi$.
	    Each terminating path $\rho$ is mapped to a single word in $\Sigma^*$, namely $\trace(\rho)$, 
		and therefore, $\pi$ yields a distribution 
		over the set of all finite words over $\Sigma$\gobble{---the alphabet of the preference DFA}
		such that for each word $w \in \Sigma^\ast$, $\Pr^\pi(w)$ is the probability that 
		$\pi$ produces $w$.
		%
		
		
	

%
	
  %
%
%
%
\begin{definition}
\label{def:weak_dominating_policies}
         Given two proper policies $\pi,\pi'$ in the terminating labeled \ac{mdp} $M$, $\pi$ \emph{weak-stochastic dominates} $\pi'$, denoted $\pi >_{\mathfrak{E}_{wk}} \pi'$, if for each $w \in \Sigma^\ast$, it holds
         that $\Pr^{\pi}(\{w\}^\uparrow) \ge \Pr^{\pi'}(\{w\}^\uparrow)$, and there exists a word $w' \in \Sigma^\ast$ such that
         $\Pr^{\pi}(\{w'\}^\uparrow) > \Pr^{\pi'}(\{w'\}^\uparrow)$.
\end{definition}

%

This definition is used to introduce the following notion.
	\begin{definition}
		\label{def:sto_dominance}
		   A proper policy $\pi \in \Pi_{prop}^M$ is \emph{weak-stochastic nondominated} if there \emph{does not exist} any policy $\pi' \in \Pi_{prop}^M$ such that $\pi' >_{\mathfrak{E}_{wk}}\pi$.
	\end{definition}
	
Informally, we say a policy $\pi$ is \emph{preferred}, if and only if it is weak-stochastic nondominated in $\Pi_{prop}^M$.
    
	Next, we state our problem informally.
	\begin{problem}
	    Given a terminating labeled \ac{mdp} and a preference model $\succeq$ over finite words $\Sigma^\ast$, compute a proper policy that is weak-stochastic nondominated.
	\end{problem}
		%
	
		\section{Main results}
		
		%
	    \subsection{Preference Deterministic Finite Automaton}
		%
	%
	 In this section, we propose a finite automaton to compatibly represent the user preferences over temporal goals. 
	 %
	%
	%
	%
	\begin{definition}
		\label{def:pdfa}
		A \ac{pdfa} for an alphabet $\Sigma$ is a tuple $\pdfa= \langle Q, \Sigma, \delta, \init, G:=(\mathbb{F}, \prefedges) \rangle$
		in which $Q$ is a finite set of states;
		$\Sigma$ is the alphabet;
		$\delta: Q \times \Sigma \rightarrow Q$ is the transition function;
		$\init \in Q$ is the initial state; 
		%
		and $G=(\prefvertices, \prefedges)$ is a \emph{preference graph} in which,
		$\prefvertices = \lbrace F_1, F_2, \cdots, F_m \rbrace$ is a partition of $Q$---i.e., 
		    $F \subseteq Q$ for each $F \in \prefvertices$, $F \cap F' = \emptyset$ for each distinct state subsets
		     $F, F' \in \prefvertices$, and $\bigcup_{F \in \prefvertices} F = Q$; and $\prefedges \subseteq \prefvertices \times \prefvertices$ is a set of directed edges.
		%
		\end{definition}  
		With a slight abuse of notation, we define the extended transition function $\delta: Q \times \Sigma^\ast \rightarrow Q$ in the usual way, \ie, $\delta(q,\sigma w) = \delta(\delta(q,\sigma), w)$ for $w\in \Sigma^\ast$ and $\sigma \in \Sigma$, and $\delta(q, \epsilon)=q$.
		%
		Note that Definition~\ref{def:pdfa} augments the classical deterministic finite automaton~\cite{hopcroft2001introduction} with the preference graph $G$, instead of a set of accepting (final) states. 
		%
	    
		
		%
		%

		For two vertices $F, F' \in \prefvertices$, we write $F  \rightsquigarrow F'$ to denote $F'$ is \emph{reachable} from $F$.
		By convention, each vertex $F$ of $G$ is reachable from itself. That is, $F\rightsquigarrow F$ always holds.

		The \ac{pdfa} encodes a preference model $\succeq$ for $\Sigma^*=(2^\calAP)^*$ as follows. 
		Consider two words $w, w' \in \Sigma^*$.
		Let $F, F' \in \prefvertices$ be the two state subsets such that $\delta(q, w) \in F $ and $\delta(q, w') \in F'$ (recall that $\prefvertices$ is a partitioning of $Q$);
		There are four cases: (1) if $F  = F'$, then $w \sim w'$; (2) if $F  \neq F'$ and $F' \rightsquigarrow F $, then $w \succ w'$; (3) if $F  \neq F' $ and $F \rightsquigarrow F'$, then $w' \succ w$; and (4) otherwise, $w \nparallel w'$.
		
		%
		
	    To illustrate, see Figure~\ref{fig:pdfa}, which shows a preference DFA specifying the preferences in the example of Figure~\ref{fig:gap_garden}.
	    State subsets $F_1 = \{q_2\}$, $F_2=\{q_4\}$, $F_3 = \{q_1\}$, and $F_4 = \{q_0, q_3, q_5\}$ respectively represent preferences $p_1$ through $p_4$.
	    %

	\gobble{For a state subset  $F_1 \in \prefvertices$, all the finite words that reach those states within $F_1$ are indifferent to one another, and for each pair of state subsets $F_1, F_2 \in \prefvertices$ for which there is a path from $F_2$ to $F_1$ in $\calG$, all words reaching a state within $F_1$ are preferred over any word reaching a state within $F_2$. If there is no path between two distinct state subsets $F_1$ and $F_2$, than any word that reaches a state within $F_1$ is incomparable to any word that reaches a state within $F_2$.}

			\begin{figure}[h]
		\centering
		\includegraphics[width=1.0\linewidth]{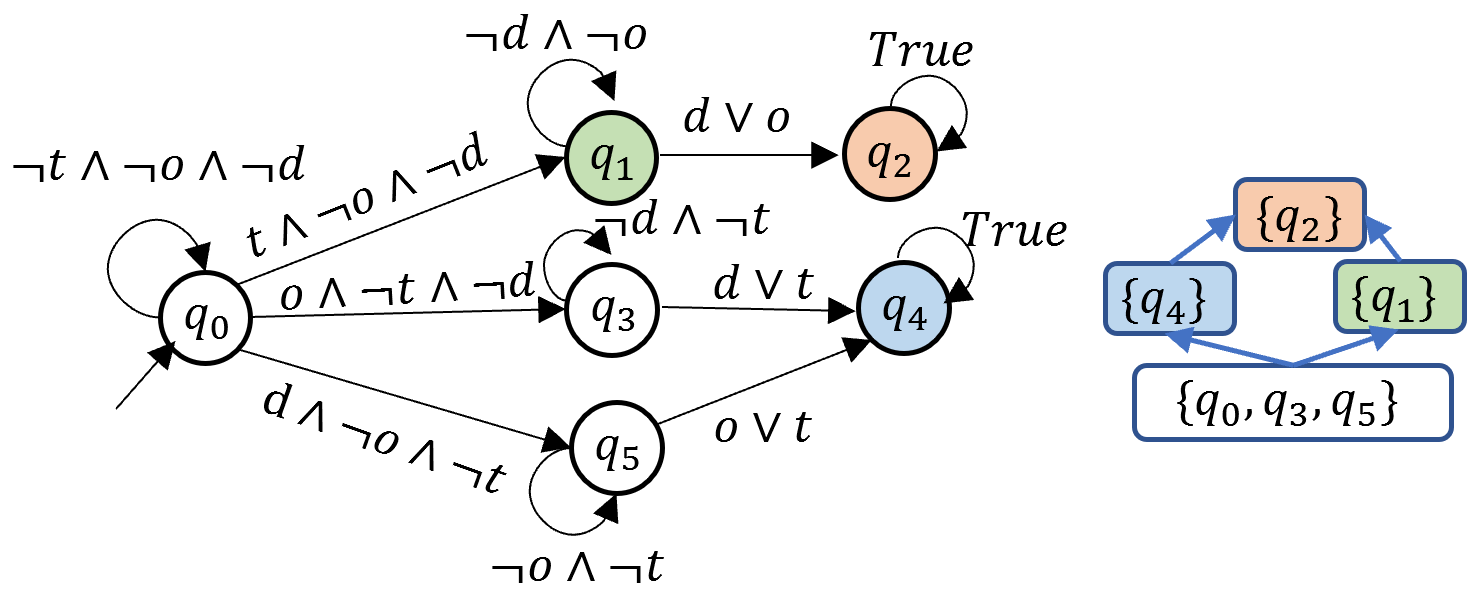}
		\caption{
		 PDFA for the example in Figure~\ref{fig:gap_garden}. 
		 \textbf{Left)} The DFA structure of PDFA.
		 \textbf{Right)} The preference graph of PDFA.
		}
		\label{fig:pdfa}
	\end{figure}
		
	
	The following Lemma allows us to define the weak-stochastic ordering over $\Sigma^\ast$, defined by the \ac{pdfa}, using its preference graph. 
	
	\begin{lemma}
	\label{lem:wUp}
	For each word $w\in \Sigma^\ast$, if $\delta(\init, w)\in F$ for some $F\in \mathbb{F}$, then 
	 \begin{multline} 
	 	\{ w\}^\uparrow = \{w' \in \Sigma^\ast \mid \exists F' \in \mathbb{F},\\ \delta(\init, w')  \in F' \text{ and } F\rightsquigarrow F'\}
	\end{multline}
\end{lemma}
The lemma directly follows from the transition function in $\pdfa$ and the transitivity property of the preference relation and thus 
the proof is omitted. 

            \problembox{Probabilistic Planning with Partially Ordered Preferences (\PPwPOP)}
{A \ac{tlmdp} $M = \langle S, A:=\Sigma_{s \in S}A_s, \mathbf{P}, \allowbreak s_0, s_\bot, \calAP, L \rangle$ and a \ac{pdfa} $\calA= \langle Q, 2^\calAP, \delta, \allowbreak \init, G := (\mathbb{F}, E) \rangle$.}
{The set of all proper polices for $M$ that are weak-stochastic nondominated under the preferences specified by $\calPA$.
}

	\section{Synthesizing a preferred policy}
	\label{sec:alg}
	We now present our algorithm\gobble{ to solve \PPwPOP}.
The first step is to augment the planning state space with the state of the \ac{pdfa}. 
With this augmented state space, we can relate the preferences over traces in the \ac{mdp} to a preference over subsets of terminating states in a product \ac{mdp} we define as follows.
	\begin{definition}[Product MDP]
	\label{def:prod}
	Let $M = \langle S, A:=\Sigma_{s \in S}A_s, \mathbf{P}, s_0, s_\bot, \calAP, L \rangle$ and 
	 $\pdfa = \langle Q, \Sigma, \delta, \init, G:=(\mathbb{F}, \prefedges) \rangle$ be respectively the \ac{tlmdp} and the \ac{pdfa}. The product of $M$ and $\pdfa$ is a tuple 	$\calM = (X, A:=\bigcup_{x \in X}A_x, \mathbf{T}, x_0, X_G, \calG := (\augnodes, \augedges))$ in which
	\begin{enumerate}
	    \item $X =  S \times Q$ is the state space;
	    \item $A$ is the action space, where for each $x=(s, q) \in X$, $A_x = A_s$ is the set of available actions at state $x$;
	    \item \label{itm:T} $\mathbf{T}: X \times A \times X \rightarrow [0, 1]$ is the transition function such that for each state $(s, q) \in X$, action $a \in A$, and state $(s', q') \in X$;
	    \begin{multline*}
     \mathbf{T}((s, q), a, (s', q')) = \\  {
    \begin{cases}
        \mathbf{P}(s, a, s') & \mbox{\text{if} $q' = \delta(q, L(s')),$} \qquad\hfill \\
        0 & \text{otherwise};
    \end{cases}
    }
\end{multline*}
	    \item $x_0 = (s_0, \delta(\init, L(s_0)))$ is the initial state;
	    \item $X_G =  \{s_\bot \} \times Q$ is the set of terminating states;
	    \item \label{itm:pgraph} $\calG = (\augnodes, \augedges)$ is the preference graph, in which, letting $W_i = \{s_\bot\}\times F_i$ for each $F_i \in \mathbb{F}$, 
	    \begin{itemize}
	        \item $\augnodes = \{W_i\mid i=1,\ldots, \abs{\mathbb{F}}\}$ is the vertex set of the graph, and
	        \item $\augedges$ is the edge set of the graph such that $(W_i, W_j)\in \augedges$ if and only if $(F_i, F_j ) \in E$.
	    \end{itemize}
	    \end{enumerate}
	\end{definition}
	%
	%
	The preference graph of this MDP has been directly lifted from the one defined for the \ac{pdfa}. We use $W \rightsquigarrow W'$ to denote 
	 that $W'$ is reachable from $W$ in the preference graph $\calG$. Again, every $W$ is reachable from itself.

    Continuing with the example in Figure~\ref{fig:pdfa}, we have $W_1 = \{ s_\bot\} \times F_1 = \{(s_\bot, q_2)\}$, $W_2 =\{s_\bot\} \times  F_2=\{(s_\bot, q_4)\}$, $W_3 = \{s_\bot \} \times F_3 = \{(s_\bot, q_1)\}$, and $W_4 = \{s_\bot\} \times F_4 = \{(s_\bot, q_0), (s_\bot, q_3), (s_\bot, q_5)\}$.

      	\normalcolor

	   
    

    %

    
    
    %
   Next, we show how to compute a weak-stochastic nondominated policy, in the sense of Definition~\ref{def:sto_dominance}, through solving a multi-objective \ac{mdp}. 
    
     Given the product MDP $\calM$ constructed  
		   in Definition~\ref{def:prod}, the \emph{weak-stochastic ordering} for
		   $\augnodes$, denoted $\mathfrak{E}_{wk}(\augnodes)$, is the family of subsets 
		\begin{equation}
	   	    \mathfrak{E}_{wk}(\augnodes) =  \{\{W\}^\uparrow \mid W\in \augnodes \}  \cup \{ \emptyset, \augnodes \}
	   \end{equation}
        where $\{W\}^\uparrow = \bigcup_{W'\in \augnodes, W\rightsquigarrow W'} \{W'\}$.	
    %

Note that by construction, the number of subsets
in $\mathfrak{E}_{wk}(\augnodes)$ minus the empty set and the set $\augnodes$ is exactly the size of $\augnodes$.   
Let $N = \abs{\augnodes}$. 
    %
    \begin{definition}[\ac{momdp}]
	\label{def:goalMDP}
	The multi-objective MDP (MOMDP) associated with the product MDP $\calM = \langle X, A, \mathbf{T}, x_0, X_G, \calG := (\augnodes, \augedges) \rangle$ in Definition~\ref{def:prod} is a tuple $\calP = \langle X, A:=\bigcup_{x \in X}A_x, \mathbf{T}, x_0, X_G, \calZ=\{Z_1, Z_2, \cdots, Z_N \} \rangle$ in which $X$, $A$, $\mathbf{T}$, $x_0$, and $X_G$ are  the same elements in $\calM$ and for each $i \in \{1, \cdots, N\}$, $Z_i = \bigcup_{W \in \{W_i\}^\uparrow} W$.
	The $i$-th objective in the \ac{momdp} is to maximize the probability for reaching the set $Z_i$.
	\end{definition}

    Note that each $Z_i$ is a subset of goal states $X_G$, and that
    the intersection of two distinct goal subsets $Z_i$ and $Z_j$ may not be empty.
    
    Using the 
running  example in Figure~\ref{fig:pdfa}, we have $\{W_1\}^\uparrow = \{W_1\}$, $\{W_2\}^\uparrow = \{W_1, W_2\}$,
    $\{W_3\}^\uparrow = \{W_1, W_3\}$, and $\{W_4\}^\uparrow = \{W_1, W_2, W_3, W_4\}$; $Z_1= W_1$, $Z_2= W_1\cup W_2$, $Z_3= W_1\cup W_3$, and $Z_4 = \bigcup_{i=1}^4 W_i$.
    
    %
    %
    %
    
    %
    
    In this \ac{momdp}, for a given randomized, finite-memory policy $\mu: X^* \rightarrow \dist(A)$, we can compute the value vector of $\mu$ as a $N$-dimensional vector $\mathbf{V}^\mu=[\mathbf{V}_1^\mu, \mathbf{V}_2^\mu, \cdots, \mathbf{V}_N^\mu]$ where for each $i$, $\mathbf{V}_i^\mu$ is the 
    probability of reaching states of $Z_i$ by following policy $\mu$, starting from the initial state.

    Given a randomized, memoryless policy $\mu: X \rightarrow \dist(A)$, to compute its value vector $\mathbf{V}^\mu$, we first set for each goal state $x_g \in X_G$, $\mathbf{V}^\mu(x_g)$ to be the vector such that for each $i \in \{1, \cdots, n \}$, $\mathbf{V}^\mu_i(x_g) = 1$ if $x_g \in Z_i$, and otherwise $\mathbf{V}^\mu_i(x_g) = 0$. Then we compute the values of the non-goals states $x \in X \setminus X_G$ via the Bellman recurrence
    \begin{equation}
        \mathbf{V}^{\mu}(x) = \sum_{a \in A} \left ( \mu(x)[a] \sum_{x' \in X} \mathbf{T}(x, a, x') \mathbf{V}^{\mu}(x') \right )
    \end{equation}
    %
    %
    
    %
    \begin{definition}
    \label{def:paretoDomPolicy}
               Given two proper polices $\mu$ and $\mu'$ for $\calM$, it is said that $\mu$ \emph{Pareto dominates} $\mu'$, 
               denoted $\mu > \mu'$, if for each $i \in \{1, \cdots, N\}$, $\mathbf{V}_i^{\mu} \geq \mathbf{V}_i^{\mu'}$, and for at least one $j \in \{1, \cdots, n\}$, $\mathbf{V}_j^{\mu} > \mathbf{V}_j^{\mu'}$.
    \end{definition}
    Intuitively, $\mu$ Pareto dominates $\mu'$ if, compared to $\mu'$, it increases the probability of reaching at least a set $Z_i$ without reducing the probability of reaching other sets $Z_j$'s.
    %
    
    %
    \begin{definition}
  \label{def:nonDomPolicy}
             A proper policy $\mu$ for the MOMDP in Definition~\ref{def:goalMDP} is \emph{Pareto optimal} if for no proper policy $\mu'$ for the MOMDP it holds that $\mu' > \mu$.
    \end{definition}
    In other words, a policy is Pareto optimal if it is not dominated by any policy.
   The \emph{Pareto front} is the set of all Pareto optimal policies. 
    %
     It is well-known that the set of memoryless policies suffices for achieving the Pareto front ~\cite{chatterjee2006markov}. Thus, we restrict to compute memoryless policies.
     %
     %
    
     %
    
    %
    

    With this in mind, we present the following result.
    \begin{theorem}
    \label{thm:pareto-weakstochastic}
        Let $\mu: X \rightarrow \dist(A)$ be a policy for $\calP$.
        Construct policy $\pi: S^* \rightarrow \dist(A)$ for the \ac{tlmdp} $M$ such that for each $\rho=s_0 s_1 \cdots s_n \in S^*$ it is set
        $\pi(\rho) = \mu(s_n, \trace(\rho))$.
        If $\mu$ is Pareto optimal, then $\pi$ is weakly-stochastic nondominated, respecting the preference specified by \ac{pdfa} $\pdfa$.
    \end{theorem}
    \begin{proof}
    
    We show that if $\mu$ is Pareto optimal then $\pi$ is weak-stochastic nondominated. To facilitate the proof, the following notation is used: Let $\Pr^\mu(\mbox{reach}(X), \calM)$ be the probability of terminating in the set $X$ given the policy $\mu$ for the MOMDP and $\Pr^\pi(\mbox{reach}(X), M)$ be the probability of terminating in the set $X$ given the policy $\pi$ in the original  \ac{tlmdp}. 
    
    

    First, consider that by the construction of the product MDP, Definition~\ref{def:prod}, preference graphs $\calG$ and $G$ are isomorphic, and thus, each $W_i \in \augnodes$ is mapped to a single $F_i \in \mathbb{F}$, and vice versa.
    Let's define $F_i^+ =\bigcup_{F, F_i \rightsquigarrow F} F$ for each $F_i \in \mathbb{F}$.
    Given that $\calG$ and $G$ are isomorphic, $W_i \rightsquigarrow W_j$ if and only if $F_i \rightsquigarrow F_j$ for all $i, j \in \{1, 2, \cdots, N \}$.
    This combined with that $Z_i = \bigcup_{W \in \{W_i\}^\uparrow} W$ for $i \in \{1, \cdots, N\}$ by Definition~\ref{def:goalMDP}, implies that for each $i$,
    \begin{equation}
    \label{eq:Zi_Fi}
       \mathbf{V}_i^\mu =  \Pr^\mu(\reach{Z_i}, \calM) = \Pr^\pi(\mbox{reach}(F_i^+) , M\}.
    \end{equation}
    
    %
    Next, for each $w, w' \in \Sigma^\ast$ such that $\delta(\init, w)=\delta(\init, w')$, it holds that $\{w\}^\uparrow = \{w'\}^\uparrow$.
    Given this and Lemma~\ref{lem:wUp}, for each $F_i$ and $w \in \Sigma^\ast$ such that $\delta(\init, w) \in F_i$, 
    \begin{equation}
    \label{eq:Fi_W}
        \Pr^\pi(\mbox{reach}(F_i^+), M) =\Pr^\pi(\{w\}^\uparrow ).
    \end{equation}
    
    Finally, given that $\mu$ is a Pareto optimal policy, by Definition~\ref{def:paretoDomPolicy} and Definition~\ref{def:nonDomPolicy}, it means there exists no policy $\mu'$ such that 
     $\mathbf{V}_i^{\mu'} \geq \mathbf{V}_i^\mu$ for all integers $1 \leq i \leq n$ and 
    $\mathbf{V}_j^{\mu'} > \mathbf{V}_j^\mu$ for some integer $1 \leq j \leq n$.
    This, by (\ref{eq:Zi_Fi}) and (\ref{eq:Fi_W}) and that the set of randomized, memoryless policies suffices for the Pareto front of $\calM$, means there exists no policy $\pi' \in \Pi^M_{prop}$ such that $\Pr^{\pi'}(\{w\}^\uparrow ) \geq \Pr^{\pi}(\{w\}^\uparrow )$ for every $w \in \Sigma^*$ and  $\Pr^{\pi'}(\{w'\}^\uparrow ) > \Pr^{\pi}(\{w'\}^\uparrow )$ for some $w' \in  \Sigma^*$.
    This, by Definition~\ref{def:weak_dominating_policies} and Definition~\ref{def:sto_dominance}, means that $\pi$ is weak-stochastic nondominated.
    \end{proof}

    Now one can use any existing methods to compute a set of Pareto optimal policies for $\calP$. 
    For a survey of those methods, see~\cite{roijers2013survey}.
    Note that computing the set of all Pareto optimal policies is generally infeasible, and thus, one needs to compute only a subset of them or to approximate them.

    \section{Case Study: Garden}
    
    %
    %
    
    In this section, we present the results from the planning algorithm for the running example in Figure~\ref{fig:gap_garden} .
    %
    
    %
    %
    %
In the garden, the actions of the robot are $N$, $S$, $E$, $W$---for receptively moving to the cell in the North, South, East, and West side of the current cell---and $T$ for staying in the current cell. The bee robot initially has a full charge, and using that charge it can fly only $12$ time steps.

\noindent \textbf{Uncertain environment: } A bird roams about  the south east part of the garden, colored yellow in the figure.
    When the bird and the bee are within the same cell, the bee needs to stop flying and hide in its current location until the bird goes away. The motion of the bird is given by a Markov chain. 
    %
    %
    %
    Besides the stochastic movement of the bird, 
the weather is also stochastic and affects the robot's planning.     
    The robot cannot pollinate a flower while raining.
    We assume when the robot starts its task, at the leftmost cell at the bottom row, it is not raining and the probability that it will rain in the next step is $0.2$. 
    This probability increases for the consecutive steps each time by $0.2$ until the rain starts.
    %
    %
    Once the rain started, the probability for the rain to stop in the five following time steps will respectively be $0.2$, $0.4$, $0.6$, $0.8$, and $1.0$, assuming the rain has not already stopped at any of those time steps.

	We implemented this case study in Python and considered two variants of it, one without stochasticity  in the robot's dynamics, and one with stochasticity. 
	In the former case, when the robot decides to perform an action to move to a neighboring cell, its actuators will guarantee with full certainty that the robot will move to that cell after performing the action.
	In the later case, the probability that the robot reaches the intended cell is $0.7$, and for each of the unintended directions except the opposite direction, the probably that the robot's actuators move the robot to that unintended direction is $0.1$. If the robot hits the boundary, it stays in its current cell.
	
	All the experiments were performed on a Windows 11 installed on a device with a core i$7$, 2.80GHz CPU and a 16GB memory. 
	\subsection{Deterministic Robot in the Uncertain Environment}
	%
	%
	The MDP for this case has $10,460$ states and $280,643$ transitions (its transition function has $280,643$ entries with non-zero probabilities).
	It took $47.38$ seconds for our program to construct the MDP.
	The product MDP had $36,649$ states and $946,467$ transitions.
	The construction time for the product MDP was $408.87$ seconds.
	%
	
	%

	 Given the preference described in Fig.~\ref{fig:pdfa}, we employ linear scalarization methods to solve the \ac{momdp}.  

Specifically,    given a    weight vector $\mathbf{w}=[w_1, w_2, w_3, w_4]$, we compute the weak-stochastic nondominated policy $\mu_{\mathbf{w}}$, by first setting $V_{\mathbf{w}}(x) = \sum_{i, x \in Z_i} w_i$ for each goal state $x \in X_G$, and then by solving the following Bellman equation for the values of  the non-goal states
    \begin{equation}
        V_{\mathbf{w}}(x) = \max_{a \in A_x} \sum_{x' \in X} \mathbf{T}(x, a, x')V_{\mathbf{w}}(x'), \forall x \in X \setminus X_G.
    \end{equation}

    The policy is recovered from $ V_{\mathbf{w}}(x)$ as
    \begin{equation}
        \mu_{\mathbf{w}}(x) = \argmax_{a \in A_x} \sum_{x' \in X} \mathbf{T}(x, a, x')V_{\mathbf{w}}(x'), \forall x\in X \setminus X_G.
    \end{equation}
	
 	We randomly generated $100$ weight vectors and used each one of them to compute a Pareto optimal policy for the MOMDP.  The computed Pareto-optimal policies in the \ac{momdp} yield  a  set of $100$ weak-stochastic nondominated policies.
	%
	%
   From the result, it is noted that none of those computed polices were weak-stochastic dominated by the other polices. This is expected due to Theorem~\ref{thm:pareto-weakstochastic}.
    %
    %
    %
    %
    Table~\ref{tbl:nondom_policies} shows $10$ out of those $100$ weight vectors along with the value vectors of the polices computed for those weight vectors and the corresponding probabilities those polices assign to the four preferences $p_1$ through $p_4$.

        	    \begin{table}[h!]
	    \setlength{\tabcolsep}{0.2em}
    \centering
    \begin{tabular}{cccc}
    \hline \hline
     & \scriptsize     Weight Vector & \scriptsize  Value Vector & \scriptsize  Prob. of individual outcomes \\ 
    \hline
   \scriptsize 1 &  \scriptsize [0.50, 0.17, 0.21, 0.12] & \scriptsize 	[0.24, 0.25, 0.98, 1.0] & \scriptsize [0.24, 0.01, 0.74, 0.01] \\
   \scriptsize 2 &  \scriptsize [0.08, 0.46, 0.38, 0.08]	 & \scriptsize [0.24, 0.42, 0.80, 1.0]		 & \scriptsize [0.24, 0.18, 0.56, 0.02] \\
    \scriptsize 3 & \scriptsize [0.73, 0.13, 0.13, 0.01]	 & \scriptsize [0.24, 0.32, 0.91, 1.0]		& \scriptsize [0.24, 0.08, 0.67, 0.01] \\
    \scriptsize 4 & \scriptsize [0.67, 0.24, 0.02, 0.07]	 & \scriptsize [0.19, 0.63, 0.51, 1.0]	& \scriptsize [0.19, 0.44, 0.32, 0.05] \\
    \scriptsize 5 & \scriptsize [0.16, 0.11, 0.04, 0.69]		 & \scriptsize [0.15, 0.71, 0.42, 1.0]			& \scriptsize [0.15, 0.56, 0.27, 0.02] \\
    \scriptsize 6 & \scriptsize [0.26, 0.16, 0.03, 0.55]		 & \scriptsize [0.15, 0.72, 0.40, 1.0]			& \scriptsize [0.15, 0.57, 0.25, 0.03] \\
    \scriptsize 7 & \scriptsize [0.24, 0.46 0.26, 0.04]				 & \scriptsize [0.17, 0.64, 0.53, 1.0]					& \scriptsize [0.17, 0.47, 0.36, 0.00] \\
    \scriptsize 8 & \scriptsize [0.22, 0.28, 0.13, 0.37]				 & \scriptsize [0.15, 0.73, 0.40, 1.0]						& \scriptsize [0.15, 0.58, 0.25, 0.02] \\
    \scriptsize 9 & \scriptsize [0.07, 0.65, 0.04, 0.25]		 & \scriptsize 	[0.00, 1.00, 0.00, 1.0]		& \scriptsize [0.00, 1.00, 0.00, 0.00] \\
    \scriptsize 10 & \scriptsize [0.18, 0.08, 0.01, 0.73]			 & \scriptsize 	[0.18, 0.63, 0.51, 1.0]			& \scriptsize [0.18, 0.45, 0.33, 0.04] \\
    \end{tabular}
        \caption{Ten weak-stochastic nondominant polices computed by our algorithm for the Garden case study. 
        %
    }
    \label{tbl:nondom_policies}
    \end{table}

    For each policy, the last column shows probability vector indicating the probability distribution over individual outcomes $p_1,\ldots, p_4$ (in this order) given the computed policy. 
    The third column shows the multi-objective value vector of each computed policy. It is noted that
    none of those value vectors dominates any other value vector.

    Rows $1$ and $3$ of this table show that even if the weight assigned to the most preferred outcome, $p_1$, is significantly higher than the weights assigned to the other preferences, the probability that $p_1$ to be satisfied is still less than $0.25$. 
    This is justified by the fact that the robot's battery capacity supports the robot for only $12$ time steps and thus to achieve $p_1$, the robot must not be stopped by the bird nor there should be raining when it reaches a cell to do pollination. The probability to satisfy these conditions given the environment dynamics is less than $0.25$.
    The probability of $p_4$ to be satisfied in any entry of this table is less than $0.05$. This is because $p_4$ has the lowest priority, and any policy would prefer to satisfy other preferences who are assigned higher priorities. 
    Although the objectives $\{p_1, p_2\}$ and $\{p_1, p_3\}$ in the first and the third rows are treated almost equally by the weight vector in terms of importance, the probability that the later to be satisfied is significantly bigger than the probability of the former to be satisfied. This is because the objective $\{p_1 \}$ contains the preference with the highest priority and that those two rows assign a very high weight to this objective, forcing the policy to try to satisfy $p_1$. Further, by attempting to perform $p_1$, the robot has the chance to accomplish $p_3$ within the same attempt, albeit if it fails to accomplish $p_1$. 
    %
    More precisely, if in attempting to perform the task $p_1$---first tulips and then at least one out of daisies and orchids---the robot succeeds to pollinate the tulips but fails to pollinate the daisies and orchids, then it has already accomplished $p_3$, even though it has failed in accomplishing what it was aiming for---$p_1$.
    %
    
    \subsection{Stochasticity in the Robot's Actions}
    %
	%
	The MDP for this variant has the same number of states, $10,460$, but it has more transitions, $779,396$, which is due to  the stochasticity in   robot's dynamics.
	The MDP construction time for this case was $279.85$ seconds and it took $2,129.02$ seconds to make the product MDP. We again computed $100$ weak-stochastic nondominated policies. 
	%
	%
	%
	
Due to the stochasticity in the robot's dynamic, we expect the policy computed for a specific weight vector
	to be less ``attractive'' than a policy computed for the same weight vector of the previous variant.
	We compare those two polices for the weight vector $[0.25, 0.25, 0.25, 0.25]$.
	The probabilities of the preferences to be satisfied for the variant without stochasticity were $[p_1: 0.24, p_2: 0.05, p_3:0.70, p_4: 0.01]$, while those probabilities for the variant with stochasticity were $[p_1: 0.01, p_2: 0.82, p_3:0.12, p_4: 0.05]$.
	%
	%
	While the former policy yields a higher probability of achieving $p_3$, the latter policy puts most of its efforts to satisfy $p_2$.

	\section{Conclusions and Future Work}
    In this paper, we proposed a finite automaton for specifying user preferences over temporal goals,
    formulated and solved a preference-based temporal planning in stochastic systems.
    The characteristic that distinguishes our work from prior work on temporal logic planning is that our formulation considers the case where the user preferences may have incomparable outcomes, and this, introduces a problem in defining how to compare policies given their distributions over outcomes as well as a problem in designing planning algorithm to solve a preferred policy. 
    We use the notion of weak-stochastic ordering to rank different policies.  
    Future work may consider other stochastic ordering that are used for ranking probability measures given a partial order over outcomes.
    %
    %
    %
    %
    %
    Another direction will be to extend this work to the planning with preference over temporal goals that are satisfied in infinite time, for instance, recurrent properties and other more general properties in temporal logic. For practical robotic applications, it would be interesting to design an interface that translates  human language preference specifications or human supervisors' feedback to a computational model, such as a preference automaton or its variant, to facilitate human-on-the-loop planning.
    
    %


	\bibliographystyle{plain}
	\bibliography{sample}
	
\end{document}